\documentclass[11pt]{article}
\pdfoutput=1
\usepackage{proceed2e}
\usepackage{fullpage}
\usepackage{verbatim}
\usepackage{blindtext}
 \usepackage{url}
\usepackage{amsthm, amsmath, amsfonts, comment, graphicx}
\usepackage[usenames,dvipsnames]{xcolor}
\usepackage{algorithm}
\usepackage{algorithmic}
\newtheorem{theorem}{Theorem}[section]
\newtheorem{lemma}[theorem]{Lemma}
\newtheorem{proposition}[theorem]{Proposition}

\newenvironment{definition}[1][Definition]{\begin{trivlist}
\item[\hskip \labelsep {\bfseries #1}]}{\end{trivlist}}

\newcommand{\x}[2]{x_{#1, #2}}
\newcommand{\calF}{\mathcal{F}}

\usepackage[compact]{titlesec}
\titlespacing{\section}{0pt}{2ex}{1ex}
\titlespacing{\subsection}{0pt}{1ex}{0ex}
\titlespacing{\subsubsection}{0pt}{0.5ex}{0ex}
\usepackage{paralist}

\begin{document}
\title{Crowdsourcing Feature Discovery via\\ Adaptively Chosen Comparisons}
\author{
{\bf James Y. Zou} \\
Microsoft Research \\
Cambridge, MA \\
\And
{\bf Kamalika Chaudhuri}  \\
University of California, San Diego          \\
San Diego, CA \\
\And
{\bf Adam Tauman Kalai}   \\
Microsoft Research \\
Cambridge, MA \\
}

\maketitle

\begin{abstract}
We introduce an unsupervised approach to efficiently discover the underlying features in a data set via crowdsourcing. Our queries ask crowd members to articulate a feature common to two out of three displayed examples. In addition we also ask the crowd to provide binary labels to the remaining examples based on the discovered features. The triples are chosen {\em adaptively} based on the labels of the previously discovered features on the data set. In two natural models of features, hierarchical and independent, we show that a simple {\em adaptive} algorithm, using ``two-out-of-three'' similarity queries, recovers all features with less labor than any nonadaptive algorithm. Experimental results validate the theoretical findings.
\end{abstract}

\section{Introduction}

\begin{figure} % This figure is designed to fit on a single column
\includegraphics[width=3in]{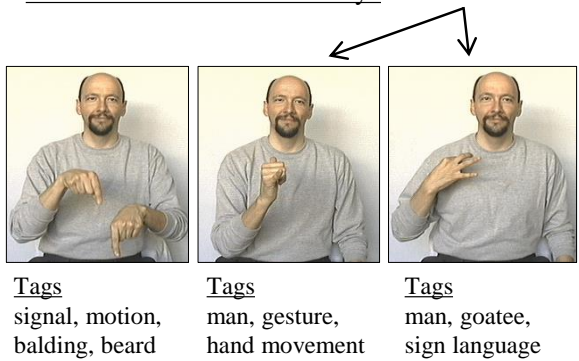}
\caption{Comparing three examples yields a useful feature whereas tagging them separately yields nondiscriminative features. }
\end{figure}

Discovering features is essential to the success of machine learning and statistics. Crowdsourcing can be used to discover these underlying features, in addition to merely labeling them on data at hand. This paper addresses the following unsupervised learning problem: given a data set, using as few crowd queries as possible, elicit a diverse set of salient, feature names along with their labels on that data set. For example, on a data set of faces, salient features might correspond to gender, the presence of glasses, facial expression, any numerous others. In this paper we focus on binary features, each of which can be thought of as a function mapping data to $\{0,1\}$. The term {\em feature name} refers to a string describing the feature (e.g., {\em male} or {\em wearing glasses}), and the {\em label} of a feature on an example refers the $\{0,1\}$-value of that feature on a that datum, as annotated by crowd workers.
Features are useful in exploratory analysis, for other machine learning tasks, and for browsing data by filtering on various facets. While the features we use are human-generated and human-labeled, they could be combined with features from machine learning, text analysis, or computer vision algorithms. In some cases, features provide a significantly more compact representation than other unsupervised representations such as clustering, e.g., one would need exponentially many clusters (such as {\em smiling white men with grey hair wearing glasses}) to represent a  small number of features.

A widely-used crowdsourcing technique for eliciting features is to simply ask people to tag data with multiple words or phrases. However, tagging individual examples fails to capture the differences between multiple images in a data set. To illustrate this problem, we asked 10 crowd workers to tag 10 random signs from an online dictionary of American Sign Language \cite{ASLdictionary}, all depicted by the same bearded man in a gray sweatshirt. As illustrated in Figure 1, the tags generally refer to his hair, clothes, or the general fact that he is gesturing with his hands. Each of the 33 tags could apply equally well to any of the 10 video snips, so none of the features could discriminate between the signs.

Inspired by prior work \cite{Patterson2012SunAttributes,crowdmedian,tamuz2011adaptively} and the familiar kindergarten question, ``which one does not belong?'', we elicit feature names by presenting a crowd worker with a triple of examples and asking them to name a feature {\em  common to any two} out of the three examples. We refer to this as a ``two-out-of-three'' or, more succinctly, 2/3 query. These features are meant to differentiate yet be common as opposed to overly specific features that capture peculiarities rather than meaningful distinctions. As shown in Figure 1, in contrast to tagging, the learned features partition the data meaningfully.

How should one choose such triples? We find that, very often, random triples redundantly elicit the same set of salient features. For example, 60\% of the responses on random sign triples distinguish signs that use one vs.\ two hands. To see why, suppose that there are two ``obvious'' complimentary features, e.g., {\em male} and {\em female}, which split the data into two equal-sized partitions and are more salient than any other, i.e., people most often notice these features first. If the data are balanced, then 75\% of triples can be resolved by one of these two features.

To address this inefficiency, once we've discovered a feature, e.g., one/two-handedness, we then ask crowd workers to label the remaining data according to this feature. This labeling is necessary eventually, since we require the data to be annotated according to all discovered features. Once we have labels for the data, we never perform a 2/3 query on {\em resolved} triples, i.e., those for which we have a feature whose labels are positive on two out of the three examples. Random 2/3 queries often result in the one of these salient features. Our adaptive algorithm, on the other hand, after learning the features of, say, ``male'' and ``female,'' always presents three faces labeled by the same gender (assuming consistent labeling) and thereby avoids eliciting the same feature again (or functionally equivalent features such as ``a man'').

The face data set also illustrates how some features are hierarchical while others are orthogonal. For instance, the feature ``bearded'' generally applies only to men, while the feature ``smiling'' is common across genders. We analyze our algorithm and show that it can yield large savings both in the case of hierarchical and orthogonal features. Proposition 4.1 states that our algorithm finds all $M$ features of a proper binary hierarchical ``feature tree'' using $M$ queries, whereas Proposition 4.2 states that any non-adaptive algorithm requires $\Omega(M^3)$ queries. The lower bound also suggests that ``generalist'' query responses are more challenging than ``specifics,'' e.g., in comparing a goat, a kangaroo, and a car, the generalist may say that the goat and kangaroo are both {\em animals} rather while the specifist may distinguish them as both {\em mammals}. We then present a more sophisticated algorithm that recovers $D$-ary trees on $M$ features and $N$ examples using $\tilde{O}(N+MD^2)$ queries, with high probability (see Proposition 4.3).

Finally, we show that in the case of $M$ independent random features, adaptivity can give an exponential improvement provided that there is sufficient data (Lemmas 5.2 and 5.3). For example, in the case of $M$ independent uniformly random features, our algorithm finds all features using fewer than $3M$ queries (in expectation) compared to a $\Omega(1.6^M)$ for a random triple algorithm. In all analysis, we do not include the cost of labeling the features on the data since this cost must be incurred regardless of which approach is used for feature elicitation. Moreover, the labeling cost is modest as workers took less than one second, amortized, to label a feature per image when batched (prior work \cite{Patterson2012SunAttributes} reported batch labeling for approximately \$0.001 per image-feature label).

Interestingly, our theoretical findings imply that 2/3 queries are sufficient to learn in both our models of hierarchial and independent features, with sufficient data. We also discuss 2/3 queries in comparison to other types, e.g., why not ask a ``1/3 query'' for a feature that distinguishes one example from two others? Note that 1/3 and 2/3 queries may seem mathematically equivalent if the negation of a feature is allowed (one could point out that two are ``not wearing green scarves''). However, research in psychology does not find this to be the case for human responses, where similarity is assumed to be based on the most common positive features that examples share (see, e.g., Tversky's theory of similariteies \cite{tversky77}). Proposition \ref{prop:lr} shows that there are data sets where larger arbitrarily large query sizes are necessary to elicit certain features.

The paper is organized as follows. After discussing related work, we define the hierarchical model in Section 2. In Section 3, we define the adaptive triple algorithm and the baseline (non-adaptive) random triple algorithm. In Section 4, we bound the number of queries of these algorithms in the case of hierarchical features. The performance under independent features are analyzed in Section 5. Section 6 considers alternative types of queries. Experimental results are presented in Section 7.

\section{Related work}

In machine learning and AI applications \cite{ParikhG11}, relevant features are often elicited from domain experts \cite{farhadi2009describing, BMVC.23.2} or from text mining \cite{berg2010automatic}. As mentioned, a common approach for crowdsourcing named features is image tagging, see, e.g., the ESP game \cite{VonAhn2004}. There is much work on automatic representation learning and feature selection from the data alone (see, e.g., \cite{bengio2013representation}), but these literatures are too large to summarize here.

One work that inspired our project was that of Patterson and Hays \cite{Patterson2012SunAttributes}, who crowdsourced nameable attributes for the SUN Database of images using comparative queries. They presented workers with random quadruples of images from a data set separated vertically and elicited features by asking what distinguishes the left pair from the right. Their images were chosen randomly and hence without adaptation. They repeated this task over 6,000 times. We discuss such left-right queries in Section \ref{sec:queries}.

For supervised learning, Parikh and Grauman \cite{ParikhG11} address multi-class classification by identifying features that are both nameable and machine-approximable. They introduce a novel computer vision algorithm to predict ``namability'' of various directions in high-dimensional space and present users with images ordered by that direction. Like ours, their algorithm adapts over time, though their underlying problem and approach are quite different. In independent work on crowdsourcing binary classification, Cheng and Bernstein \cite{flock} elicit features by showing workers a random pair of positive and negative example. They cluster the features using statistical text analysis which reduces redundant labeling of similar features (which our algorithm does through adaptation), but it does not solve the problem that a large number of random comparisons are required in order to elicit fine-grained features. They also introduce techniques to improve the feature terminology and clarify feature definitions, which could be incorporated into our work as well.

Finally, crowdsourced feature discovery is a human-in-the-loop form of unsupervised dictionary learning (see, e.g., \cite{sparseCoding}).  Analogous to the various non-featural representations of data, crowdsourcing other representations has also been studied. For hierarchical clustering, a number of algorithms have been proposed (see, e.g., Chilton {\em et al} \cite{chilton2013cascade}). Also, Kernel-based similarity representations have been crowdsourced adaptively as well \cite{tamuz2011adaptively}.

\section{Preliminaries and Definitions\label{sec:prelim}}

We first assume that there is a given set $X = \{ x_1, x_2, \ldots, x_N \}$ of examples (images, pieces of text or music, etc.) and an unknown set $\mathcal{F}=\{f_1,f_2,\ldots f_M\}$ of binary features $f_j:X \rightarrow \{0,1\}$ to be discovered. We say that feature $f_j$ is {\em present} in an example $x_i \in X$ if $f_j(x)=1$, {\em absent} if $f_j(x)=0$, and we abuse notation and write $\x{i}{f}\equiv f(x_i)$ and $\x{i}{j}\equiv f_j(x_i)$. Hence, since there are $M$ hidden features and $N$ examples, then there is an underlying latent $N$-by-$M$ feature allocation matrix $A$ with binary entries. The $i$th row of $A$ corresponds to sample $x_i$, and the $j$th column of $A$ corresponds to feature $f_j$.

Our goal is to recover this entire matrix $A$, together with names for the features, using minimal human effort.
\begin{definition}
Given a feature $f$ and an example $x_i$, a \textbf{labeling query} $L(x_i, f)$ returns $f(x_i)$.
\end{definition}

As we will discuss, in practice labeling is performed more efficiently in batches. A consideration for query design is that we want each contrastive query to be as cognitively simple as possible for the human worker. Our analysis suggests that comparisons of size three suffice, but for completeness we define comparisons on pairs as well.

\begin{definition}
A \textbf{2/3 query} $Q(x, y, z)$ either returns a feature $f \in \mathcal{F}$ such that $f(x)+f(y)+f(z)=2$ or it returns \textit{NONE} if no such feature exists.

A \textbf{1/2 query} on $Q(x, y)$ either returns a feature $f \in \mathcal{F}$ such that $f(x)+f(y)=1$ or returns \textit{NONE} if $x$ and $y$ are identical.
\end{definition}
We also refer to 2/3 queries as {\em triple queries} and 1/2 queries as {\em pair queries}. Note that we can simulate a pair query $Q(x, y)$ by two triple queries $Q(x,x,y)$ and $Q(x, y, y)$. We say that a feature $f$ \textbf{distinguishes} a set of examples $S$ if $\sum_{x \in S} f(x)=|S|-1$, i.e., it holds for all but one example in $S$.

\begin{figure} % This figure is designed to fit on a single column
\includegraphics[width=3in]{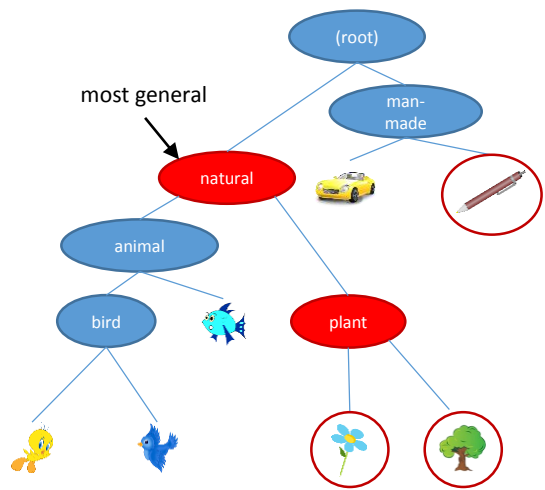}
\caption{A sample proper binary feature tree. When comparing the pen, flower, and tree, the distinguishing features are {\em natural} and {\em plant}. A generalist would respond with {\em natural}. \label{fig:taxonomy}}
\end{figure}

\begin{definition}
A query is \textbf{resolved} if there is a known distinguishing feature for the query, or it is known that \textit{NONE} is the outcome of the query.
\end{definition}

Algorithm~\ref{alg:adaptrip}, the Adaptive Triple Algorithm, is the main algorithm we use for experimentation and analysis (though we also analyze a more advanced algorithm.

\section{Hierarchical Feature Models\label{sec:hier}}
We now consider the setting where the features and examples form a tree, with each internal node (other than the root) corresponding to a single feature and each leaf corresponding to a single example. The features that are 1 for an example are defined to be those on the path to the root, and the others are 0. The root is not considered a feature. Hence, if feature $f$ is an ancestor of $g$, then $g\leq f$ in that whenever $g$ is 1, $f$ must be 1 as well.

\begin{algorithm}
\begin{algorithmic}[1]
\REQUIRE Examples $X = \{x_i\}$.
\ENSURE A set of features $F=\{f\}$ and their corresponding labels on all examples $\x{i}{f}$ for $i \leq N, f \in F$.
\STATE Randomly select a triple $\{x,y,z\}$ from the set of all unresolved triple queries. Let $f = Q(x,y,z)$.
\STATE If $f\neq $NONE: (a) add it to $F$, (b) run the labeling query $L(x_i, f)$ for all $x_i \in X$, and (c) update the set of unresolved queries.
\STATE If all all triples of examples can be resolved by one of the discovered features, terminate and output $F$ and the labels. Otherwise, go to 1.
\end{algorithmic}
\caption{Adaptive Triple}\label{alg:adaptrip}
\end{algorithm}

\begin{definition}
A {\em feature tree} $T$ is a rooted tree in which each internal node (aside from the root) corresponds to a distinct feature and each leaf corresponds to a distinct example. The value of a feature on an example is 1 if the node corresponding to that feature is on the the path to the root from the leaf corresponding to the example, and 0 otherwise.
\end{definition}
Note that our algorithms recover the features but not the tree explicitly -- reconstructing the corresponding feature tree is straightforward if the data is consistent with one.

\subsection{Binary feature trees}

In this section, we consider the standard notion of {\em proper binary trees} in which each internal node has exactly two children. Figure \ref{fig:taxonomy} illustrates a proper binary feature tree.
\begin{proposition}
For a proper binary feature tree on $M$ features, the Adaptive Triple algorithm finds all features in $M$ queries. \label{prop:binadaptive}
\end{proposition}
\begin{proof}

To prove this proposition, we will show that: (a) we never receive a \textit{NONE} response in the Adaptive Triple Algorithm, and (b) every feature has at least one triple for which it is the unique distinguishing feature. Since a query in this algorithm cannot return an already discovered feature, and since there are $M$ features, this implies that there must be exactly $M$ queries.

For (a), let $f$ be the least common ancestor of an example triple $\{x,y,z\}$. Since $T$ is proper, $f$ must have exactly two children. By the definition of least common ancestor, two out of $\{x,y,z\}$ must be beneath one child of $f$ (call this child $g$) while the other one is beneath the other child. Then $g$ is a distinguishing feature for $Q(x,y,z)$. Hence, we should never receive a NONE response. 

For (b), observe that every internal node (other than the root) has at least one triple for which it is the {\em unique} distinguishing feature. In particular, given any internal node, $f$, let $l$ and $r$ be its left and right children. Let $x$ and $y$ be examples under $l$ and $r$ (with possibly $x=l$ or $y=r$ if $l$ or $r$ are leaves).  Let $s$ be the sibling of $f$ (the other child of its parent) and let $z$ be any leaf under of $s$ (again $z=s$ if $s$ is a leaf). Then it is clear that $f$ is the unique distinguishing triple for $x$, $y$, and $z$. For example, in Figure \ref{fig:taxonomy}, for the feature {\em plant}, a triple such as the flower, tree, and fish, would uniquely be distinguished by {\em plant}.
\end{proof}

Now consider different ways to answer queries: define a {\em generalist} as an oracle for $Q$ that responds to any query with the shallowest distinguishing feature, i.e., the one closest to the root. For example, given the pen, flower and tree of Figure \ref{fig:taxonomy}, the generalist would point out that the flower and tree are both {\em natural} rather than that they are both {\em plants}. Also, say an algorithm is {\em non-adaptive} if it specifies its queries in advance, i.e., the triples cannot depend on the answers to previous queries but could be random.
We also assume that the data is {\em anonymous} which means that we can think of the specific examples being randomly permuted in secret before being given to the algorithm.

We now show that any general-purpose non-adaptive algorithm that does not exploit any content information on the examples requires at least $\Omega(M^2)$ examples to find all $M$ features and at least $\Omega(M^3)$ if all queries are answered by generalists.

\begin{proposition}
If the examples correspond to a random permutation of the leaves of a proper binary tree $T$ with $M$ features, then any non-adaptive algorithm requires at least $M^2/12$ queries to recover all $M$ features with probability 1/2. Furthermore, if queries are answered by generalists, then any non-adaptive algorithm requires at least $M^3/24$ queries to find all features with probability 1/2.
\end{proposition}
Figure \ref{fig:taxonomy} sheds light on this proposition -- in order to discover the feature bird, we mush choose both birds in a triple. If the queries are answered by a generalist, we would have to choose the birds and fish. The probability of choosing two specific examples is $O(1/M^2)$ while the probability of choosing three specific examples is $O(1/M^3)$.
\begin{proof}
Let $f$ be the deepest feature (or one of them if there are more than one). Let $f$ have children $x$ and $y$ which must be leaves since $f$ is a deepest internal node. Let $s$ be the sibling of $f$.  By assumption $x$ and $y$ are leaves. Now, in order to discover $f$, the triple must consist of $x$ and $y$ and another node, which happens with probability $(N-2)/{N \choose 3}=\frac{6}{N(N-1)}< 6/M^2$ for a random triple (since $N=M+2$). By the union bound, if there are only $M^2/12$ triples, it will fail to discover $f$ with probability at least $1/2$.

Now consider a generalist answering queries. Let $S$ be the set of leaves under $s$. Since $f$ is the deepest feature, $S$ must be a set of size 1 or 2 depending on whether or not $s$ is a leaf. It is not difficult to see that the only triples that return $f$ (for a generalist) are $x,y$ and an element of $S$. Hence there are at most 2 triples that recover $f$. Since there are ${N \choose 3}> M^3/6$ triples, if there are fewer than $M^3/24$ triples, then the probability that any one of them is equal to one of the two target triples is at most 1/2. The union bound completes the proof.
\end{proof}

Note that pairs are insufficient to recover internal nodes in the case where a specifist answers queries. This motivates the need for triples; moreover, Proposition~\ref{prop:binadaptive} shows that triple queries suffice to discover all the features in a binary feature tree.

\begin{algorithm}[t!]
\begin{algorithmic}[1]
\REQUIRE Examples $X = \{x_i\}$ and an exploration parameter $\theta$.
\ENSURE The set of features $F$ and labels for all examples $\x{i}{f}$.
\STATE Query pairs of examples until we have, for each pair, found a feature that distinguishes them, or determined that they have identical features (by direct comparison or transitivity).
\STATE Maintain a queue $Q$ of features to explore, and a queue of already discovered features $F$. Initialize $Q = \{r\}$, where $r$ is a default \textit{root} feature defined as: $x_{ir} = 1, \forall i \in X$. Initialize $F = \{\}$.
\WHILE{Queue $Q$ is not empty}
\STATE Pop a feature $f$ from $Q$. Set $\mbox{off}(f) = \{f_j \mbox{ s.t. } \not\exists f' \mbox{ with } f_j < f' < f \}$. Represent each feature $f_j$ in  $\mbox{off}(f)$ by a randomly selected example $x_j$ such that $x_{j,f_j} = 1$.
\STATE Uniformly randomly select distinct examples $x, y, z \in \mbox{off}(f)$, and query $\{x,y, z\}$. If the query returns a feature $f'$, push $f'$ to $Q$, run labeling queries $\{x, f'\}$ for all $x \in \mbox{off}(f)$ and update $\mbox{off}(f)$.
\STATE If Step 5 returns $\theta$ consecutive \textit{NONE}s, then add $f$ to $F$ and go to Step 4 and pop the next feature from the $Q$.
\ENDWHILE
\RETURN $F$ and the labels $\x{i}{f}$.

\end{algorithmic}
\caption{Adaptive Hybrid}\label{alg:adaphybrid}
\end{algorithm}

\subsection{General feature trees}
We now present a theoretical algorithm using triple queries which allows us to efficiently learn general ``$D$-ary leafy feature trees,'' which we define to be a feature tree in which: (a) every internal node (i.e., feature) has at most $D$ internal nodes (but arbitrarily many leaves) as children, and (b) no internal node has a single child which is an internal node. Condition (a) is simply a generalization of the standard branching factor of a rooted tree, and condition (b) rules out any ``redundant'' features, i.e., features which take the same value for each example.

\begin{proposition} [Adaptive Hybrid, Upper Bound]
Let $T$ be a $D$-ary leafy feature tree with $N$ examples and $M$ features. The Adaptive Hybrid algorithm with exploration time $\theta = 3D^2 \log \frac{M}{\delta}$ terminates after $O(N + MD^2\log \frac{M}{\delta})$ number of triple queries and finds all features with probability $\geq 1-\delta$.
\label{lem:Dary}
\end{proposition}

The proof of Proposition~\ref{lem:Dary} makes use of the following Lemma.

\begin{lemma}
Let $T$ be a non-star, $D$-ary leafy feature tree. Then the Random Triple algorithm finds at least one feature with probability $\geq 1-\delta$ using $3D^2\log \frac{1}{\delta}$ queries.
\label{lem:Darybasic}
\end{lemma}

Due to space limitations, the proofs are deferred to the appendix.

\section{Independent features}\label{sec:indep}

In this section we consider examples drawn from a distribution in which different features are independent. Consider a statistical model in which there is a product distribution $D$ over a large set of examples $\mathcal{X}$. This model is used to represent features that are independent of one other. An example of two independent features in the Faces data set might be ``Smiling'' and ``Wearing Glasses.'' We assume that $D$ is a product distribution over $M$ independent features. Thus $D$ can be described by a vector $\{ p_f, f \in \calF \}$, where for any feature $f \in \calF$, $p_f = \Pr_{x \sim D} \bigl[f(x) = 1\bigr]$. We also abuse notation and write $p_i$ for $p_{f_i}$. We assume $0<p_i<1$.

In this model, there is a concern about how much data is required to recover all the features. In fact, for certain features there might not even be any triples among the data which elicit them. To see this, consider a {\em homogenous crowd} that all answers queries according to a fixed order on features. Formally, if more than one feature distinguishes a triple, suppose the feature that is given is always the distinguishing feature $f_i$ of smallest index $i$. Intuitively, this models a situation where features are represented in decreasing salience, i.e., differences in the first feature (like gender) are significantly more salient than any other feature, differences in the second feature stand out more than any feature other than the first, and so forth. Now, also suppose that all features have probability 1/2 of being positive.
\begin{lemma}
If $p_1=p_2=\cdots=p_M=1/2$, then with a homogeneous crowd, $N\geq 1.1^M$ examples are required to find all features with probability $1/2$ even if all triples are queried.
\end{lemma}
\begin{proof}
Since $p_i=1/2$, the probability of any feature distinguishing a triple is $3/8$. Therefore, a homogenous crowd will only output the last, least salient feature if it the only distinguishing feature, which happens with exponentially small probability $(3/8) (5/8)^{M-1}$ for a random triple. Given $N<1.1^M$ examples, there $N^3<1.1^{3M}$ triples. By the union bound, with probability less than $(3/8) (5/8)^{M-1} 1.1^{3M} < 1/2$ will any of them elicit the last feature.
\end{proof}
On the other hand, we show that all features will be discovered with a finite number of samples. In particular, say a feature $f$ is {\em identifiable} on a data set if there exists a triple such that $f$ is the unique distinguishing feature. If it is identifiable, then of course the adaptive triple algorithm will eventually identify it. We now argue that, given sufficiently many examples, all features will be identifiable with high probability.
\begin{lemma}[Identifiability in the Independent Features Model]
Suppose $N$ examples are drawn iid from the Independent Features Model where feature $f$ has frequency $p_f$. For any feature $f$, let:
\[ \quad \tau_f = 3p_f^2 (1 - p_f) \prod_{g \neq f} \left(1 - p_g^2 (1 - p_g)\right). \]
Moreover, let $\tau_{\min} = \min_f \tau_f$.  If $N \geq \Omega(\log (1/\tau_{\min})/\tau_{\min})$, then, with constant probability, all features are identifiable by triple queries.
\label{lem:independent1}
\end{lemma}
The above exponential upper and lower bounds are worst case. In fact, it is not difficult to see that for a totally {\em heterogeneous crowd}, which outputs a random distinguishing feature, if all $p_i=1/2$, only $N=O(\log M)$ examples would suffice to discover all features because one could query multiple different people about each triple until one discovered all distinguishing features. Of course, in reality one would not expect a crowd to be completely homogeneous nor completely heterogeneous (nor completely generalists nor completely specifists), and one would not expect features to be completely independent or completely hierarchical. Instead, we hope that our analysis of certain natural cases helps shed light on why and when adaptivity can significantly help.

As we now turn to the analysis of adaptivity and the number of queries, we make a ``big data'' assumption that we have an unbounded supply of examples. This makes the analysis simple in that the distribution over unresolved triples takes a nice form. We show that the number of queries required by the adaptive algorithm is linear in the number of features, while it grows exponentially with the number of features for any non-adaptive algorithm.

We first provide an upper bound on the number of queries of the Adaptive Triple algorithm in this model.
\begin{lemma}[Adaptive Triple]
Suppose for $j = 1, \ldots, k$, we have $M_j$ independent features with frequency $p_j$ and infinitely many examples. Then the expected number of queries used by Adaptive Triple  to discover all the features is at most $\sum_{j=1}^{k} \frac{M_j}{q_j}$,
%\[ \sum_{j=1}^{M} \min\left( \frac{d_j}{q_j}, d_j + \frac{1 - q_j}{q_j^2} \right), \]
where $q_j = 3p_j^2 (1 - p_j)$. For the Adaptive Pair algorithm, set $q_j = 2p_j(1 - p_j)$.
\label{lem:ubindependent}
\end{lemma}

We next provide lower bounds on the number of queries of any non-adaptive algorithm under the independent feature model.

\begin{lemma}[non-adaptive triple]
Suppose for $j = 1, \ldots, k$, we have $M_j$ independent features with frequency $p_j$ and infinitely many examples. Let $q_j = 3p_j^2 (1 - p_j)$. The expected number of queries made by any non-adaptive triple algorithm is at least:
\[ \frac{1 - q_{\max}}{ \prod_{i=1}^{k} (1 - q_i)^{M_i}}, \]
where $q_{\max} = \max_i q_i$. 
\label{lem:random}
\end{lemma}

To interpret these results, consider the simple setting where all the features have the same probability: $p_j = p$. Then the random triple algorithm requires \textit{at least} $1/(1-q)^{M-1}$ queries on average to find all the features. This is exponential in the number of features, $M$. In contrast, the adaptive algorithm \textit{at most} $M/q$ queries on average to find all the features, which is only linear in the number of features.

%We use two kinds of queries for eliciting features -- contrastive queries and labeling queries. In a contrastive query, we select two subsets of examples $S_1, S_2 \subset X$ and ask a human to provide a feature that is in all of the examples in $S_1$ but not in any of the examples in $S_2$. We can choose the sizes of $S_1$ and $S_2$, though we prefer them to be small to make it cognitively easier for the human to answer the query. In particular, we can set $S_2 = \emptyset$ in which case the user returns a feature that is common across $S_1$.

\section{Other types of queries}\label{sec:queries}

Clearly 2/3 queries are not the only type of queries. For example, an alternative approach would use 1/3 queries in which one seeks a feature that distinguishes one of the examples from the other two. 
Such queries could result in features that are very specific to one image and fail to elicit higher-level salient features. 
%For instance, consider a triple of men with two wearing glasses and one wearing a green scarf. A common 2/3 answer is ``wearing glasses'' whereas a natural 1/3 answer could be ``wearing a green scarf.'' 
Under the hierarchical feature model, 1/3 queries alone are not guaranteed to discover all the features.

A natural generalization of the left-vs-right queries in previous work \cite{Patterson2012SunAttributes,flock} are queries with sets $L$ and $R$ of sizes $|L|\leq \ell, |R|\leq r$, where a valid answer is a feature common to all examples in $L$ and is in no examples in $R$. We refer to such a query as an $\ell-r$ query $L-R$. In fact, a 2/3 query on $\{x,y,z\}$ may be simulated by running the three L-R queries $\{x,y\}-\{z\}$, $\{y,z\}-\{x\}$, and $\{x,z\}-\{y\}$. (Note that this may result in a tripling of cost, which is significant in many applications.) There exist data sets for which L-R queries can elicit all features (for various values of $\ell,r$) while 2/3 queries may fail.

\begin{proposition}\label{prop:lr}
For any $\ell, r \geq 1$, there exists a data set $X$ of size $N=|X|=\ell+r$ and a feature set $\mathcal{F}$ of size $M=|\mathcal{F}|=1+\ell+r$ such that $\ell-r$ queries can completely recover all features while no $\ell'-r'$ query can guarantee the recovery the first feature if $\ell'<\ell$ or if $r' < r$.
\end{proposition}
\begin{proof}
Let the examples be $X=L\cup R$ where $L=\{x_1,x_2,\ldots,x_\ell\}$ and $R=\{x'_1,\ldots,x'_r\}$. Let $\mathcal{F}=\{f\}\cup G\cup H$ where the feature $f$ satisfies $f(x)=1$ if $x\in L$ and $f(x)=0$ if $x\in R$. Define the features $g_1,g_2,\ldots, g_\ell \in G$ to be $g_i(x)=1$ for all $x\in L\setminus \{x_i\}$ and $g_i(x_i)=0$, otherwise. Define $H=\{h_1,\ldots, h_r\}$ where $h_j(x)=0$ for all $x \in R\setminus \{x'_j\}$ and $h_j(x)=1$, otherwise. It is clear that the query $L-R$ necessarily recovers $f$, the query $\emptyset-\{x_i\}$ recovers $g_i$, and the query $\{x'_j\}-\emptyset$ recovers $h_j$. Moreover, for any query $L'-R'$ with $x_i \not\in |L'|$, it is clear that $g_i$ is as good an answer as $f$. Conversely, if $x'_j\not \in R'$, then clearly $h_j$ is as good an answer as $f$. Hence, if the feature $f$ is ``least salient'' in that other features are always returned if possible, no $\ell'-r'$ query will recover $f$.
\end{proof}

\section{Experiments}

\begin{figure*}[t!]
\centering
\includegraphics[trim=1cm 2.5in 1cm 0.5in, width=6in]{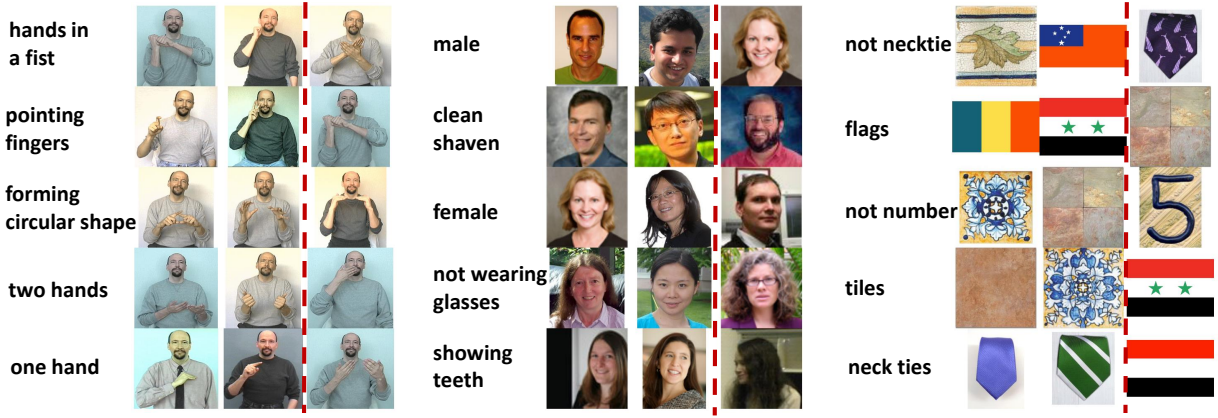}
\caption{The first five features obtained from a representative run of the Adaptive Triple algorithm on the signs (left), faces (middle) and products (right) datasets. Each triple of images is shown in a row beside the proposed feature, and the two examples declared to have that feature are shown on the left, while the remaining example is shown on the right.}\label{fig:face}
\end{figure*}

We tested our algorithm on three datasets: 1) a set of 100 silent video snips of a sign-language speaker \cite{ASLdictionary}; 2) a set of 100 human face images used in a previous study \cite{tamuz2011adaptively}; 3) a set of 100 images of ties, tiles and flags from that same study \cite{tamuz2011adaptively}.
All the images and videos were initially unlabeled. The goal was to automatically elicit features that are relevant for each dataset and to label all the items with these features. We implemented our Adaptive Triple algorithm on the popular crowdsourcing platform, Amazon Mechanical Turk, using two types of crowdsourcing tasks. In a \textit{feature elicitation} task, a worker is shown three examples and is asked to specify a feature that is common to two of the examples but is not present in the third. In a \textit{labeling} task, a worker is shown one feature and all examples and is asked which examples have the feature. To reduce noise, we assigned each labeling task to five different workers, assigning each label by majority.

To compare adaptivity to non-adaptivity, we implemented a {\bf Random Triple} algorithm that picks a set of random triples and then queries them all. To compare triples to pairs, we also implemented an {\bf Adaptive Pair} algorithm, defined in the analogous way to the random triple algorithm except that it only does pair queries.

The Adaptive Triple algorithm automatically determines which sets of examples to elicit features from and which combination of example and feature to label.
Figure~\ref{fig:face} shows the first five queries of the Adaptive
 Triple algorithm from one representative run on the three datasets. For example, on the \textit{face} data, after having learned the broad gender features \textit{male} and \textit{female} early on, the algorithm then chooses all three female faces or all three male faces to avoid duplicating the gender features and to learn additional features.

We compared the Adaptive Triple Algorithm to several natural baselines: 1) a non-adaptive triple algorithm that randomly selects sets of three examples to query; 2) the Adaptive Pairs algorithm; 3) the standard tagging approach where the worker is shown one example to tag at a time and is asked to return a feature that is relevant for the example. We used two complementary metrics to evaluate the performance of these four algorithms: the number of \textit{interesting and distinct} features the algorithm discovers, and how efficiently can the discovered features partition the dataset.

In many settings, we would like to generate as many distinct, relevant features as possible. On a given data set, we measure the distance between two features by the fraction of examples that they disagree on (i.e. the Hamming distance divided by the number of examples). We say that a feature is \textit{interesting} if it differs from the all 0 feature (a feature that is not present in any image) and from the all 1 feature (a features that is ubiquitous in all images) in at least 10\% of the examples. A feature is \textit{distinct} if it differs in at least 10\% of the examples from any other feature. If multiple features are redundant, we represent them by the feature that was discovered first.

\begin{table}
\resizebox{\columnwidth}{!}{%
\begin{tabular}{llll}
  & signs & faces & products  \\
  \hline
{\bf adaptive triple} & {\bf 24.5 (3.8}) & {\bf 25.3 (0.3)} & {\bf 19 (1.4)}   \\
random triples	& 12.5 (0.4) & 18.7 (2.7) & 14 (1.4)   \\
adaptive pairs 	&11.5 (1.1) & 14.5  (1.8) & 10.5 (0.4)  	\\
tagging			& 9 (0.4) & 13 (0.71)	& 12 (0.4) \\
\end{tabular}}
\caption{Number of interesting and distinct features discovered. Standard error shown in parenthesis. }

\label{table:number}
\end{table}

Table~\ref{table:number} shows the number of interesting and distinct features discovered by the four algorithms. On each dataset, we terminate the algorithm after 35 feature elicitation queries. Each experiment was done in two independent replicates--different random seeds and Mechanical Turk sessions. The Adaptive Triple algorithm discovered substantially more features than all other approaches in all three datasets. The non-adaptive approaches (random triples and tagging) were hampered by repeated discoveries of a few obvious features--one/two-handed motions in signs, male/female in faces and product categories in products. Once Adaptive Triples learned these obvious features, it purposely chose sets of examples that cannot be distinguished by the obvious features in order to learn additional features. Adaptive comparison of pairs of example performed poorly not because of redundant features but because after it learned a few good features, all pairs of examples can be distinguished and the algorithm ran out of useful queries to make. This is in agreement with our analysis of hierarchical features. Pairwise comparisons are only guaranteed to find the base-level features of the hierarchy while triples can provably find all the features.

To evaluate how efficiently the discovered features can partition the dataset, we compute the average size of the partitions induced by the first $k$ discovered features. More precisely, let $f_t$ be the $t$th discovered feature. Then features $f_1, ..., f_k$ induces a partition on the examples, $P_1, ..., P_R$, such that examples $x_i, x_j$ belong to the same partition if they agree on all the features $f_1, ..., f_k$. The average fraction of indistinguishable images is $g(\{f_1, ..., f_k\}) = \sum_r |P_r|^2/N^2$. Before any feature were discovered, $g = 1$. If features perfectly distinguish every image, then $g = 1/N$.

\begin{figure}
\centering
\includegraphics[trim=2cm 3in 2cm 0cm, width=3in]{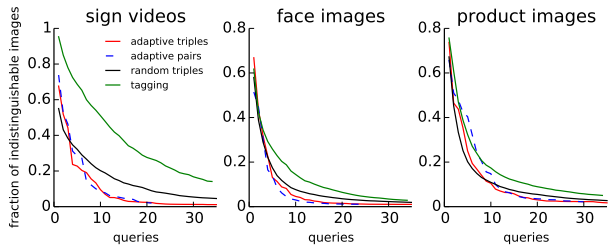}
%\vspace{-0.5cm}
\caption{Comparisons of the adaptive triple algorithm with benchmarks.}\label{fig:scattering}
\end{figure}

In Figure~\ref{fig:scattering}, we plot the value of $g$ for the adaptive triple algorithm and the benchmarks as a function of number of queries. The adaptive algorithms requires significantly fewer queries to scatter the images compared to the non-adaptive algorithms. On the sign data set, for example, the adaptive triple required 13 queries to achieve $g = 0.05$ (i.e. a typical example is indistinguishable from 5\% of examples), while the random triples required 31 queries to achieve the same $g = 0.05$.  Adaptive Triples and Adaptive Pairs both achieved rapid decrease in $g$, indicating that both were discovering good discriminatory features. However, as we saw above, Adaptive Pairs terminated early because it no longer had any unresolved pairs of examples to query, while Adaptive Triples continued to discover new features.

\section{Discussion}

We have introduced a formal framework for modeling feature discovery via comparison queries.
Consistent with previous work \cite{Patterson2012SunAttributes}, we demonstrated that tagging can be inefficient for generating features that are diverse and discriminatory. Our theoretical analysis suggested that the Adaptive Triple algorithm can efficiently discover features, and our experiments on three data sets provided validation for the theoretical predictions.
Moreover, unlike previous non-adaptive feature elicitation algorithms which had to detect redundant features (either using humans or natural language processing), our algorithm is designed to avoid generating these redundant features in the first place.

A key reason that our algorithm outperformed the non-adaptive baseline is that in all three of our data sets there were some features that were especially salient, namely gender for faces, one or two hands for sign language, and product type for products. A interesting direction of future work would be to investigate the performance of adaptive algorithms in other types of data.

Our analysis suggests that homogeneous crowds and crowds of generalists should be most challenging for eliciting features. 
Modeling the salience of features and the diversity of the crowd are also interesting directions of future work. In particular, our algorithm made no explicit attempt to find the most salient features, e.g., one could imagine aggregating multiple 2/3 responses to find the most commonly mentioned features. In addition, one could leverage the fact that different users find different features to be salient and model the diversity of the crowd to extract even more features. 

\bibliography{bib}

\begin{thebibliography}{10}

\bibitem{ASLdictionary}
Asl dictionary.
\newblock \url{http://www.lifeprint.com/dictionary.htm}.

\bibitem{bengio2013representation}
Yoshua Bengio, Aaron Courville, and Pascal Vincent.
\newblock Representation learning: A review and new perspectives.
\newblock {\em Pattern Analysis and Machine Intelligence, IEEE Transactions
  on}, 35(8):1798--1828, 2013.

\bibitem{berg2010automatic}
Tamara~L Berg, Alexander~C Berg, and Jonathan Shih.
\newblock Automatic attribute discovery and characterization from noisy web
  data.
\newblock In {\em Computer Vision--ECCV 2010}, pages 663--676. Springer, 2010.

\bibitem{flock}
Justin Cheng and Michael Bernstein.
\newblock Flock: Hybrid crowd-machine learning classifiers.
\newblock In {\em Proceedings of CSCW}, 2015.

\bibitem{chilton2013cascade}
Lydia~B Chilton, Greg Little, Darren Edge, Daniel~S Weld, and James~A Landay.
\newblock Cascade: Crowdsourcing taxonomy creation.
\newblock In {\em Proceedings of the SIGCHI Conference on Human Factors in
  Computing Systems}, pages 1999--2008. ACM, 2013.

\bibitem{farhadi2009describing}
Ali Farhadi, Ian Endres, Derek Hoiem, and David Forsyth.
\newblock Describing objects by their attributes.
\newblock In {\em Computer Vision and Pattern Recognition, 2009. CVPR 2009.
  IEEE Conference on}, pages 1778--1785. IEEE, 2009.

\bibitem{crowdmedian}
Hannes Heikinheimo and Antti Ukkonen.
\newblock The crowd-median algorithm.
\newblock In {\em First AAAI Conference on Human Computation and
  Crowdsourcing}, 2013.

\bibitem{sparseCoding}
Honglak Lee, Alexis Battle, Rajat Raina, and Andrew~Y Ng.
\newblock Efficient sparse coding algorithms.
\newblock In {\em Advances in neural information processing systems}, pages
  801--808, 2006.

\bibitem{ParikhG11}
Devi Parikh and Kristen Grauman.
\newblock Interactively building a discriminative vocabulary of nameable
  attributes.
\newblock In {\em CVPR}, pages 1681--1688. IEEE, 2011.

\bibitem{Patterson2012SunAttributes}
Genevieve Patterson and James Hays.
\newblock Sun attribute database: Discovering, annotating, and recognizing
  scene attributes.
\newblock In {\em Proceeding of the 25th Conference on Computer Vision and
  Pattern Recognition (CVPR)}, 2012.

\bibitem{tamuz2011adaptively}
Omer Tamuz, Ce~Liu, Ohad Shamir, Adam Kalai, and Serge~J Belongie.
\newblock Adaptively learning the crowd kernel.
\newblock In {\em Proceedings of the 28th International Conference on Machine
  Learning (ICML-11)}, pages 673--680, 2011.

\bibitem{tversky77}
Amos Tversky.
\newblock Features of similarity.
\newblock {\em Psychological Review}, 84:327--352, 1977.

\bibitem{VonAhn2004}
Luis von Ahn and Laura Dabbish.
\newblock Labeling images with a computer game.
\newblock In {\em CHI '04: Proceedings of the SIGCHI conference on Human
  factors in computing systems}, pages 319--326, New York, NY, USA, 2004. ACM.

\bibitem{BMVC.23.2}
Josiah Wang, Katja Markert, and Mark Everingham.
\newblock Learning models for object recognition from natural language
  descriptions.
\newblock In {\em Proceedings of the British Machine Vision Conference}, pages
  2.1--2.11. BMVA Press, 2009.
\newblock doi:10.5244/C.23.2.

\end{thebibliography}
\bibliographystyle{plain}

\clearpage
\onecolumn
\appendix
\section{Analysis of the hierarchical feature model}

\begin{proof} (Proof of Lemma~\ref{lem:Darybasic})
Let $(x_i, x_j, x_k)$ be any triplet of examples. Let $f_{lca}$ be the lowest common ancestor of $x_i$, $x_j$ and $x_k$ in $T$; that is, $f_{lca}$ is the lowest feature $f$ in $T$ such that $\x{i}{f} = \x{j}{f} = \x{k}{f} = 1$. If $f_{lca}$ is also the lowest common ancestor of any two out of $( x_i, x_j, x_k )$, then the query $\{x_k, x_j, x_i\}$ will return \textit{NONE}; otherwise it returns a node feature.

Recall that in the Adaptive Hybrid algorithm, after the double queries in step 1, we associate each feature $f_j$ with a single example. Thus, for the rest of the proof we assume that there exists a one-to-one mapping between an example and a feature at the leaf node of $T$.

Let $f$ be any feature in $T$, and let $L_f$ be the subset of triples $(x, y, z)$ such that $f$ is the lowest common ancestor to $x$, $y$ and $z$. For any triple $(x, y, z)$, let $I(x, y, z) = 1$ if one of the triple queries $\{x, y, z\}$ returns a feature; otherwise let $I(x, y, z) = 0$. The total number of triple queries which will return a feature can be written as: $\sum_{f \in T} \sum_{(x, y, z) \in L_f} I(x, y, z)$.

Suppose that $f$ has $k$ children in $T$. Let $n_1 \geq n_2 \ldots \geq n_k$ be the number of examples associated with these children. We have two cases.

In the first case, $n_1 \geq 2$. Let us call such a feature \textit{heavy}. In this case, querying any triple $(x, y, z)$ where $x$ and $y$ are from the first child will result in a feature. The fraction of such triples in $L_f$ is at least $\frac{n_1 (n_1 - 1)}{\sum_i n_i (\sum_i n_i - 1)} \geq \frac{1}{2D^2}$. Thus, for a heavy $f$, $\sum_{(x, y, z) \in L_f} I(x, y, z) \geq \frac{|L_f|}{2 D^2}$.

In the second case, $n_1 = 1$. Call such an $f$ a \textit{light} feature. As $T$ is not a star, there exists at least one leaf $l \in T$ which does not have $f$ as an ancestor. Consider triples of the form $(x, y, l)$ where $x$ and $y$ are descendants of $f$ such that $f$ is their lowest common ancestor, and let $S_{f, l}$ be the set of all such triples.

It turns out that $S_{f, l}$ has some nice properties. First, $|S_{f, l}| \geq |L_f|/D$; this is because if $f$ has $k \leq D$ children, then, $|S_{f, l}| = \binom{k}{2}$ while $|L_f| = \binom{k}{3}$.  Second, if $(x, y, l)$ is a triple in $S_{f, l}$, then the queries $(x, y, l)$ will return a new feature. Finally, suppose we map each light feature $f$ to the set $S_{f, l}$; then, the sets $S_{f, l}$ and $S_{f', l}$ are disjoint when $f \neq f'$.

Therefore,
\[ \sum_{\text{light\ } f} |L_f|  \leq \sum_{\text{light\ }f} D |S_{f, l}| \leq D \sum_{f} \sum_{(x, y, z) \in L_f} I(x, y, z) \]
Combining the two cases, we get:
\[ \sum_{\text{light\ } f} |L_f| + \sum_{\text{heavy\ }f} |L_f| \leq (D + 2D^2) \sum_{f} \sum_{(x, y, z) \in L_f} I(x, y, z) \leq (3D^2) \sum_f \sum_{(x, y, z) \in L_f} I(x, y, z) \]

Therefore, if we draw a random triple of examples from the subtree below $f$, and make the corresponding three triple queries, the probability that we get a new feature is $\geq \frac{1}{3D^2}$. The lemma follows.
\end{proof}

\begin{proof} (Proof of Proposition~\ref{lem:Dary} )
We begin by observing that any time the queries in Step 5 of the algorithm return a feature, it must be a new feature that we haven't seen before.

Each leaf feature $f$ is the unique solution to the double query $\{x, y\}$, where $x$ is under $f$ and $y$ is under a sibling leaf feature. Thus, all the leaf features are identified by double queries. Moreover, the double queries return at most $N$ \textit{NONE}  answers.

Let $f$ be the feature that we have currently popped from the queue $Q$, i.e. the feature that we are currently exploring. Let $T_f$ be the induced subtree of $T$ with root at $f$ and leaves the set $\mbox{off}(f)$. Note that $T_f$ is the true underlying subtree (that is, not the subtree that we have found), and it is also $D$-ary.  The Adaptive Hybrid algorithm now randomly samples triples of examples from $\mbox{off}(f)$ to query. If $T_f$ is a star, then there are no new features to be found and this subroutine stops after $\theta$ queries. Otherwise, when $\theta = O(D^2 \log \frac{M}{\delta})$, from Lemma~\ref{lem:Darybasic}, with probability $\geq 1 - \frac{\delta}{M}$, it returns a new feature with high probability. Therefore the probability of finding all $M$ features is $\geq 1 - \delta$.
The algorithm terminates after $O(N+MD^2 \log \frac{M}{\delta})$ total queries.
\end{proof}

\section{Analysis of the independent feature model}

\begin{proof} (Proof of Lemma~\ref{lem:independent1})
Let $f$ be any feature, and let $x_i$ and $x_j$ be a randomly drawn pair of examples from $D$. The probability that $f$ satisfies the double query $(x_i, x_j)$ is $2p_f (1 - p_f)$; moreover, the probabilty that $f$ is the only feature that satisfies this query is $\Delta_f =  2p_f (1 - p_f) \prod_{g \neq f} \left(1 - 2p_g (1 - p_g)\right)$.

Now consider the process of drawing $N/2$ pairs of random examples from $D$. The probability that the $i$-th pair $(x, y)$ is such that the double query $(x, y)$ is uniquely satisfied by feature $f$ is $\Delta_f$. The first part of the lemma follows from a coupon collector's argument. The proof of the second part is very similar.
\end{proof}

\begin{proof} (Proof of Lemma~\ref{lem:ubindependent})
Let $f_j$ be a feature with frequency $p_j$, and let $(x_1, x_2, x_3)$ be a randomly drawn triple of examples. The probability that $f_j$ satisfies the triple query $(x_1, x_2, x_3)$ is $q_j = 3p_j^2 (1 - p_j)$. Let $\mathcal{F}$ be the full set of features. Suppose we have already seen the set of features $S$. Then the probability that the next query will discover an unseen feature is at least: $1 - \prod_{j \in \mathcal{F} \setminus S}(1 - q_j)$, and therefore the expected time to discover the next unseen feature is at most:
\[ \frac{1}{1 - \prod_{j \in \mathcal{F} \setminus S}(1 - q_j)} \]
This quantity is an decreasing function of $q_j$. Thus, the worst case order of discovering features that maximizes the expected discovery time is from high to low values of $q_j$.

WLOG we will assume that $q_1 \geq q_2 \geq \ldots \geq q_M$. The total expected discovery time is at most:
\[
\sum_{j=1}^M \frac{1}{1-\prod_{i \geq j} (1-q_i)} \leq \sum_{j = 1}^M \frac{1}{ 1 - (1-q_j)}
\]

\begin{comment}
\[ \sum_{j=1}^{k} \sum_{i=1}^{d_j} \frac{1}{1 - (1 - q_j)^{d_j - i} \prod_{i' \geq j} (1 - q_i')^{d_{i'}}} \leq \sum_{j=1}^{k} \sum_{i=1}^{d_j} \frac{1}{1 - (1 - q_j)^{d_j - i}} \]

Since $1 - (1 - q_j)^{d_j - i} \geq q_j$, the first part of the lemma follows. For the second part, observe that for any $j$,
\begin{eqnarray*}
 \sum_{i=1}^{d_j} \frac{1}{1 - (1 - q_j)^{d_j - i}} & \leq & \sum_{i=1}^{d_j} \left(1 + q_j^{d_j - i} + q_j^{2(d_j - i)} + q_j^{3(d_j - i)} + \ldots \right) \\
& \leq & d_j + \sum_{i \geq 1} i q_j^i \leq d_j + \frac{1 - q_j}{q_j^2}
\end{eqnarray*}
The lemma thus follows.
\end{comment}
\end{proof}

\begin{proof} (Proof of Lemma~\ref{lem:random})
Suppose that the features are discovered according to some order $\pi$. Then, the probability a random triple elicits the last feature $i$ is:
\[ \prod_{i < N} (1 - q_{\pi(N)}) q_{\pi(N)} \]
Of course this is minimized when $q_{\pi(N)}$ is minimized. Although a general adaptive algorithm can have a structure to the triples it chooses, we can use the union bound to argue to bound the probability that any triple elicits the last feature. In this case, each triple is essentially random.
\end{proof}

\end{document}